\documentclass[twocolumn]{article}
\usepackage[numbers, compress]{natbib}
\title{Entropic Optimal Transport in Random Graphs}
\date{}
\usepackage{graphicx}
\usepackage{xcolor}
\usepackage{amsthm, amsmath, amsfonts, amssymb}
\usepackage{mystyle}
\usepackage{hyperref}
\usepackage[margin=.9in]{geometry}
\usepackage{relsize}
\usepackage{parskip}
\usepackage{authblk}

\author{Nicolas Keriven\thanks{CNRS \& GIPSA-lab. 11 rue des Math\'ematiques, 38400 St-Martin-d’Her\`es, France. \url{firstname.name@cnrs.fr}\\
This work was partly supported by ANR GRandMa ANR-21-CE23-0006.
NK thanks G. Peyré for many fruitful discussions.}}

\begin{document}

\maketitle

\begin{abstract}
  In graph analysis, a classic task consists in computing similarity measures between (groups of) nodes. In latent space random graphs, nodes are associated to unknown \emph{latent variables}. One may then seek to compute distances \emph{directly in the latent space}, using only the graph structure. In this paper, we show that it is possible to consistently estimate entropic-regularized Optimal Transport (OT) distances between groups of nodes in the latent space.
  We provide a general stability result for entropic OT with respect to perturbations of the cost matrix. We then apply it to several examples of random graphs, such as graphons or $\epsilon$-graphs on manifolds. Along the way, we prove new concentration results for the so-called Universal Singular Value Thresholding estimator, and for the estimation of geodesic distances on a manifold.
\end{abstract}

\section{Introduction}

Graphs are becoming increasingly popular to represent structured data in machine learning \citep{Hu2020}, such as social or physical networks, proteins interaction networks, molecules, 3D meshes, and so on. Given a large graph, a classic task in graph analysis consists in computing some \emph{similarity measure} between nodes or groups of nodes, e.g. for clustering or edge prediction purposes, or to define proximity-based node embeddings \citep{Rossi2020}.

More precisely, the framework we consider here is the following: the user observes a graph $G$ with nodes numbered $\{1,\ldots,N\}$, chooses two groups of target nodes $\{i_1,\ldots, i_n\}$ and $\{j_1,\ldots,j_m\}$ as well as potential weights over them $\alpha \in \RR_+^n, \beta \in \RR_+^m$, and wants to compute some distance between these two groups of weighted nodes.
For instance, in a social network context, one might select two groups of people according to some criterion (e.g. geographical location), and desire to know how ``close'' they are in terms of their underlying (unknown) preferences. On a 3D mesh, one may want to compare two entire regions with respect to the geodesic distance of the underlying manifold \citep{Peyre2010}.

There are many metrics to compare individual nodes on graphs, such as the shortest path distance, resistance distance \citep{Klein1993}, random walk-based distances \citep{Lovasz1993,Camby2018}, or a metric between node embeddings \citep{Grover2016, Goyal2017}. In this paper, we assume that the graphs are generated by \textbf{latent space random graph (RG) models} \citep{Hoff2002, Tang2013b, Yin2013, Smith2019}, that is, each node is associated with an unknown latent variable $z_i \in\RR^d$, and edges are randomly drawn between nodes with ``similar'' variables, as measured by a \emph{connectivity kernel} (Sec.~\ref{sec:background}). Unsurprisingly, in many situations, metrics between nodes approximate ``true'' metrics in the latent space \cite{Bernstein2000}. The present paper extend this idea to \textbf{Optimal Transport} (OT) distances between \emph{groups} of nodes.

Many situations in machine learning require to compare groups of points, or more generally discrete distributions over groups of points \cite{Peyre2019}. Optimal Transport \citep{Villani2008, Peyre2019} has recently gained in popularity due to its flexibility, interpretable geometric properties, and computationally efficient algorithms. In particular, \cite{Cuturi2013a} showed that fast solvers, such as the so-called Sinkhorn's algorithm, can be used to solve \textbf{entropic-regularized} OT, which additionally enjoys good statistical properties \cite{Genevay2018, Mena2019}.
Computing OT distances between groups of points $\{x_1,\ldots, x_n\}$ and $\{y_1,\ldots, y_m\}$ (or distributions supported on them) requires the knowledge of a \emph{cost matrix} $C = [c(x_i,y_j)]_{ij}$ between each pairs of points, where $c$ is a function that indicates how costly it is to ``transport'' mass from $x_i$ to $y_j$.
In the context of random graphs, the $x_i,y_j$ are the unknown latent variables of the nodes of the graph, and the true cost $c(x_i,y_j)$ is unknown. This paper will thus examine the stability of OT distance to having only access to \emph{noisy estimates} $\hat C$ of the cost matrix.

\paragraph{Outline.} In this paper, we show that OT distances between latent variables can be consistently estimated from large RGs, or in other words: if one applies the OT methodology to some cost matrix derived from the graph, and if the graph follows an RG model, then one is in fact estimating a ``true'' underlying, much more interpretable, OT distance. We start with some background in Sec.~\ref{sec:background}, then prove a generic OT stability result in Sec. \ref{sec:stability}. We examine two classic, but fairly different, settings: RGs with fixed connectivity kernels between latent variables (referred to as ``non-local''), or kernels whose connectivity radius vanishes as the number of nodes increases (``local'', see Sec.~\ref{sec:background}). For the latter, we show in Sec.~\ref{sec:local} how using shortest paths in the graph leads to estimation of the OT distance, where the cost is the geodesic distance on the underlying manifold supporting the latent variables. In the former, we show in Sec.~\ref{sec:nonlocal} that the so-called Universal Singular Value Thresholding (USVT) estimator \cite{Chatterjee2015} leads to consistent estimation of the true OT cost for relatively sparse random graphs (degrees grow logarithmically). We also show how a Gaussian-like connectivity kernel can lead to faster concentration, for a very specific cost and regularization parameter. Some numerical illustrations are provided along the way, the code is available at \url{github.com/nkeriven/otrg}. 

\paragraph{Related work.} 
There is a vast literature on similarity measure between graph nodes \cite{Rossi2020}. Besides classical fixed distances such as the shortest-path distance or resistance distance \cite{Klein1993, VonLuxburg2014, Camby2018}, recent methods are often based on \emph{node embeddings} \cite{Goyal2017, Rossi2020}, that is, each node is associated to a vector such that some metric between them is meaningful with respect to some criterion or downstream task \cite{Grover2016, Velickovic2019}. In light of this, there has been substantial work on \emph{learning} node embeddings, most recently using Graph Neural Networks (GNNs), see \cite{Wu2019a, Hamilton2020, Goyal2017, Rossi2020}. In contrast, here we study a ``non-learned'' OT-based distance between discrete distributions on groups of nodes.

For RGs with non-local kernels, latent variables and pairwise distance estimation have a long history \cite{Hoff2002, Araya2019, Yin2013, Chatterjee2015}, which includes the vast field of community detection on Stochastic Block Models (SBM) \cite{VonLuxburg2008, Lei2015, Abbe2018}. Moreover, it is known that the Laplacian of the graph converges (in some sense) to an integral operator \cite{VonLuxburg2008, Rosasco2010, Lovasz2012}, which allows the study of the convergence of node embedding methods such as GNNs \cite{Keriven2020a, Keriven2021}.
For local kernels, shortest paths on meshed point clouds are generally associated with geodesic distances on manifolds \cite{Bernstein2000, Alamgir2012, Hwang2016, Davis2019}, which have numerous applications in shape analysis \cite{Peyre2010}. In this case, the Laplacian converges to a Laplace-Beltrami differential operator \cite{Belkin2007a, GarciaTrillos2019}, and OT can be used as a theoretical tool in this context to prove convergence results of some variational problems on graphs \cite{GarciaTrillos2016}. To our knowledge, this paper is however the first to make a direct, simple connection between OT and RGs.

Optimal Transport on graphs is often associated with the \emph{Gromov-Wasserstein distance} \cite{Memoli2011, Memoli2014}, an OT-based distance to compare \emph{different graphs} (and more generally metric spaces). On the contrary, here we exploit the random graph model to estimate an OT distance \emph{between nodes} of the same graph. Our results are based on a stability bound with respect to the matrix $C$ (Thm. \ref{thm:stability}). While the stability of optimization problems is a vast topic \cite{Bonnans2000}, there are surprisingly few works about this in OT, besides a few smoothness results \cite{Chen2016b}. Instead, several papers in computational OT seek to ``robustify'' OT by jointly optimizing the estimation of the cost matrix and the OT distance \cite{Carlier2020, Dhouib2020, Lin2020}. Here, we show instead how directly plugging an estimator $\hat C$ in Sinkhorn's algorithm leads to a consistent estimation of entropic OT.

\paragraph{Notations} For a matrix $M$, we denote by $M_{i:j, k:l}$ the rectangular submatrix with rows from $i$ to $j$ and columns from $k$ to $l$. For a function $f:\RR \to \RR$ and a matrix $M$, $f(M)$ is the matrix with $f$ applied entrywise, such as for instance $e^M$. We denote by $\odot$ the Hadamard (entrywise) product between vectors. The norm $\norm{\cdot}$ is the Euclidean norm for vectors, and the operator norm for matrices $\norm{M} = \sup \norm{Mx}/\norm{x}$. The norm $\norm{\cdot}_\infty$ is the maximal absolute value of the entries, for both vectors and matrices, and $\norm{M}_F \eqdef \sqrt{\sum_{ij} M_{ij}^2}$ is the Frobenius norm. The probability simplex is defined as $\Delta^n = \{x \in \RR_+^n; 1^\top x=1\}$. For a norm $\norm{\cdot}$, balls with radius $r$ are denoted by $\Bb_r^{\norm{\cdot}} = \{v \in \Vv; \norm{v}\leq r\}$, where the space $\Vv$ is always clear from the context.

\section{Background}\label{sec:background}

We start with some background material on entropic OT and latent space random graphs.

Let $\alpha \in \Delta^n,\beta \in \Delta^m$ be two discrete distributions and $C \in \RR_+^{n \times m}$ a \emph{cost matrix}. Usually, $\alpha$ and $\beta$ are associated to two sets $\{x_1,\ldots, x_n\}$ and $\{y_1,\ldots, y_m\}$, and $C$ is defined as $C_{ij} = c(x_i,y_j)$ for a certain cost function $c$. For $\epsilon \geq 0$, the \emph{regularized OT distance} \cite{Cuturi2013a} is defined as
\begin{align}
  \label{eq:primalOT}\tag{$\mathcal{P}_\epsilon$}
  &\Ww_\epsilon^C(\alpha, \beta) \eqdef \min_{P\in \Pi(\alpha, \beta)} \Ww_\epsilon^C(\alpha, \beta, P) \\
  &\quad\text{with } \Ww_\epsilon^C(\alpha, \beta, P) \eqdef \ps{P}{C} + \epsilon \text{KL}\pa{P | \alpha \otimes \beta} \notag
\end{align}
where $\Pi(\alpha, \beta) = \{P \in \RR_+^{n \times m} ~|~ P1_m = \alpha, P^\top 1_n = \beta\}$, $\ps{P}{C} = \sum_{ij} P_{ij} C_{ij}$ and $\text{KL}\pa{P | \alpha \otimes \beta} = \sum_{ij} P_{ij} \log\pa{\frac{P_{ij}}{\alpha_i \beta_j}}$ with the convention that $\text{KL}\pa{P | \alpha \otimes \beta}= +\infty$ if $\alpha_i \beta_j =0$ and $P_{ij}>0$. The ``normal'' (non-regularized) OT distance is obtained for $\epsilon=0$ \cite{Peyre2019}. For $\epsilon>0$, the problem \eqref{eq:primalOT} is strictly convex and the minimizer is unique, here denoted by $P^{C,\alpha, \beta}$ or $P^C$ for short. It is also known \cite{Peyre2019} that this optimal coupling has the form $P^C_{ij} = u_i K_{ij} v_j$ for some vector $u\in\RR^n$, $v\in\RR^m$ and $K = e^{-C/\epsilon}$. It can be found efficiently by the celebrated Sinkhorn's algorithm \cite{Cuturi2013a}, stochastic approaches \cite{Genevay2016}, block-coordinate ascent \cite[Chap.~4]{Peyre2019}, and so on. Note that, it is also possible to consider the KL divergence with respect to the uniform measure, however here we adopt the version found in \cite{Genevay2016} for normalization purposes.

A \emph{latent space} random graph \cite{Hoff2002} with adjacency matrix $A$ on $N$ nodes is generated as follows. To the nodes are associated latent variables $\{z_1, \ldots, z_N\} \subset \RR^d$, often \emph{unknown} and/or \emph{random}, and unweighted random edges are drawn independently as Bernoulli variables:
\begin{equation}\label{eq:random_edges}
  \forall i<j:\quad a_{ij} \sim \text{Bernoulli}(w_N(z_i, z_j))
\end{equation}
for some \emph{connectivity kernel} $w_N:\RR^d \times \RR^d \to [0,1]$.
The kernel $w_N$ is allowed to vary with the number of nodes $N$ to adjust the \textbf{sparsity} of the graph, that is, the ratio between the number of edges and the number of nodes. The most two common cases are: $w_N(z,z') = \rho_N w(z,z')$ with fixed kernel $w$ and sparsity-inducing coefficient $\rho_N \to 0$ , which we refer to as \textbf{non-local kernels} \cite{Lei2015, Araya2019}, and $w_N(z,z') = 1_{\norm{z-z'} \leq h_N}$ with decreasing radius $h_N \to 0$, referred to as \textbf{local kernels} \cite{GarciaTrillos2019}. The former includes Erdös-Rényi graphs, SBMs, graphons \cite{Lovasz2012}, and are routinely used for instance in social network analysis \cite{Hoff2002, Goldenberg2009}. The latter is usually called $\epsilon$-graphs\footnote{Here we use $h$ instead of $\epsilon$ to denote the connectivity radius to avoid conflict with the regularization parameter.}, and are popular in shape analysis \cite{Peyre2010}.

As we mentioned in the introduction, our settings are the following: a random graph with $N$ nodes is observed, and the user chooses two groups of target nodes of size $n$ and $m$ with associated weights $\alpha, \beta$ that they want to compare. Since the nodes of a graph can be arbitrarily re-ordered, without lost of generality we assume that the target nodes are respectively the first $n$ and the following $m$ nodes. We denote the corresponding unknown latent variables by $x_i=z_i$ for $i\leq n$ and $y_j = z_{n+j}$ for $j\leq m$, such that the full set of latent variables is $\{x_1,\ldots, x_n, y_1, \ldots, y_m, z_{n+m+1}, \ldots, z_N\}$. In random graphs models, the limit case is obtained when the number of nodes $N$ grows to $\infty$, which may also be the case of $n$ and $m$ at certain rates that are made explicit in each of our results.

\section{Stability of Regularized OT}\label{sec:stability}

In this section, we derive generic results guaranteeing stability of $\Ww^C_\epsilon$ when modifying the cost matrix $C$. First observe that it is easy to show the following.

\begin{proposition}\label{prop:stability_infty}
  For all $\epsilon\geq 0$, we have
  \begin{equation}\label{eq:stability_infty}
    \abs{\Ww_\epsilon^C(\alpha, \beta) - \Ww_\epsilon^{\hat C}(\alpha, \beta)} \leq \norm{C - \hat C}_\infty
  \end{equation}
\end{proposition}

\begin{proof}
  Denoting by $P^C$ any minimizer of $\Ww_\epsilon^C(\alpha, \beta, P)$, we have
  \begin{align*}
    &\Ww_{\epsilon}^{\hat C}(\alpha,\beta) = \Ww_\epsilon^{\hat C}(\alpha,\beta,P^{\hat C}) \\
    &\quad \leq \Ww_\epsilon^{\hat C}(\alpha,\beta,P^C) \leq \Ww^C_{\epsilon}(\alpha,\beta) \\
    &\qquad\quad+ \sup_{P \in \Pi(\alpha,\beta)}\abs{\Ww_\epsilon^{\hat C}(\alpha,\beta,P) - \Ww_\epsilon^{C}(\alpha,\beta,P)}
  \end{align*}
  and vice-versa, so
  \begin{align*}
    &\abs{\Ww_{\epsilon}^{\hat C}(\alpha,\beta) - \Ww^C_{\epsilon}(\alpha,\beta)} \\
    &\quad\leq \sup_{P\in \Pi(\alpha,\beta)}\abs{\Ww_\epsilon^{\hat C}(\alpha,\beta,P) - \Ww_\epsilon^{C}(\alpha,\beta,P)} \\
    &\quad= \sup_P\abs{\ps{C-\hat C}{P}}\leq \norm{C-\hat C}_\infty
  \end{align*}
  since for any $P \in \Pi(\alpha,\beta)$ we have $\sum_{ij} P_{ij} = 1$.
\end{proof}

Hence, for any level of regularization, the two OT distances are close to each other as soon as all the individual elements of $C-\hat C$ are. However, as we will see, in some situations convergence of $\norm{C-\hat C}_\infty$ will not hold, and one would rather prefer a bound involving more ``global'' norms such as the Mean Square Error $\frac{1}{nm}\norm{C-\hat C}_F^2$.
Nevertheless, when $\epsilon=0$, it is known that the minimizing OT plans $P_C$ are generally sparse. Due to this fact, the bound \eqref{eq:stability_infty} is generally the ``best'' that we can hope for. Fortunately, it is known that when $\epsilon$ is \emph{strictly positive}, the transport plan is \emph{not} sparse \cite{Peyre2019}. Moreover, as mentioned before, when $\epsilon>0$ the minimizing OT plan is unique and the cost function is strongly convex. Following this, the next theorem is our first main result.

\begin{theorem}\label{thm:stability}
  Define $c_{\max},c_{\min}$ such that $0\leq c_{\min} \leq C_{ij},\hat C_{ij} \leq c_{\max}$ for all $i,j$. For all $\epsilon>0$, it holds that:
  \begin{align}
    &\abs{\Ww_\epsilon^C(\alpha, \beta) - \Ww_\epsilon^{\hat C}(\alpha, \beta)} \notag \\
    &\quad \leq \epsilon e^{(2c_{\max}-c_{\min})/\epsilon} \norm{\alpha}\norm{\beta} \norm{e^{-C/\epsilon}-e^{-\hat C/\epsilon}} \label{eq:stability_spectral}
  \end{align}
  Moreover, denoting by $P^C$ and $P^{\hat C}$ the minimizers in \eqref{eq:primalOT} for $C$ and $\hat C$ respectively, we have
  \begin{align}
    &\KL(P^C|P^{\hat C})\leq \epsilon^{-1} e^{2(c_{\max}-c_{\min})/\epsilon}\norm{\alpha}\norm{\beta}\norm{C-\hat C}_F \notag\\
    &~ + e^{(4c_{\max} - 7c_{\min}/2)/\epsilon} \sqrt{\norm{\alpha}\norm{\beta} \norm{e^{-C/\epsilon}-e^{-\hat C/\epsilon}}} \label{eq:stability_TP}
  \end{align}
\end{theorem}

As expected, all the bounds above are insensitive to shifting the costs $C \to C + a$ and $\hat C\to \hat C + a$ since this shifts $W^C_\epsilon$ by the same quantity and leaves the minimizing OT plan unchanged. 
Note that \eqref{eq:stability_spectral} uses the spectral norm between the $K$, while \eqref{eq:stability_TP} includes the Frobenius between the $C$. The latter is actually strictly worse: by virtue of the mean value theorem and $\norm{\cdot}\leq \norm{\cdot}_F$, the following proposition is immediate.
\begin{proposition}\label{prop:spec2F}
  Under the assumption of Thm.~\ref{thm:stability},
  \begin{equation}
    \norm{e^{-C/\epsilon}-e^{-\hat C/\epsilon}} \leq \epsilon^{-1} e^{-c_{\min}/\epsilon} \norm{C-\hat C}_F
  \end{equation}
\end{proposition}
Most often the terms $\alpha_i,\beta_j$ are ``balanced'', i.e. of the order of $1/n$ and $1/m$, and we look for estimators $\hat C$ such that $\frac{1}{\sqrt{nm}} \norm{e^{-C/\epsilon}-e^{-\hat C/\epsilon}}$ or the MSE $\frac{1}{\sqrt{nm}} \norm{C-\hat C}_F$ converges, the second being stronger than the first. Moreover, in that case we also have $\frac{1}{\sqrt{nm}} \norm{C-\hat C}_F \leq \norm{C-\hat C}_\infty$, confirming that the entropic bounds are (up to potentially large multiplicative constants) better than the unregularized one \eqref{eq:stability_infty}.
In the next sections, we give three examples that use various versions of our bounds.

\begin{proof}[Proof of Theorem \ref{thm:stability}]
  We work with the dual of \eqref{eq:primalOT}. For \emph{any} matrix $K \in \RR_+^{n \times m}$, we define:
  \begin{equation}
    \Ll_\epsilon^K(\alpha,\beta) \eqdef \max_{f \in \RR^n, g \in \RR^m} \Ll^K_\epsilon(f,g,\alpha,\beta) \label{eq:dualOT}\tag{$\mathcal{D}_\epsilon$}
  \end{equation}
  where $\Ll^K_\epsilon(f,g,\alpha,\beta)\eqdef\alpha^\top f + \beta^\top g - \epsilon (e^{f/\epsilon} \odot \alpha)^\top K (e^{g/\epsilon} \odot \beta) + \epsilon$. When $K = e^{-C/\epsilon}$, we have $\Ll_\epsilon^K(\alpha,\beta) = \Ww_\epsilon^C(\alpha, \beta)$, and in this case the optimal dual potentials $(f,g)$ relates to the optimal $(u,v)$ by $u = \alpha \odot e^{f/\epsilon}$ and $v = \beta \odot e^{g/\epsilon}$ \cite{Peyre2019}.
  The following lemma, proved in the appendix, is the key to proving Theorem \ref{thm:stability}. It shows that the properties of the matrix $K$ allow to bound the optimal dual potentials.

  \begin{lemma}\label{lem:dualbound}
    Assume that $K$ is such that $0< \delta_{\min} \leq K_{ij} \leq \delta_{\max} \leq  1$ for all $i,j$. Then there are optimal potentials of \eqref{eq:dualOT} that satisfy $\norm{f}_\infty, \norm{g}_\infty \leq \epsilon \log(\sqrt{\delta_{\max}}/\delta_{\min})$.
  \end{lemma}

  We denote by $f^K, g^K$ optimal solutions of \eqref{eq:dualOT} that satisfy these bounds.
  In light of Lemma \ref{lem:dualbound}, for a constant $\eta>0$, we define the following optimization problem, which is just \eqref{eq:dualOT} with added box constraints:
  \begin{equation}\label{eq:dualOTbounded}
    \Ll^K_{\epsilon,\eta}(\alpha,\beta) \eqdef \max_{f,g\in \Bb_{\epsilon \log\eta}^{\norm{\cdot}_\infty}} \Ll^K_\epsilon(f,g,\alpha,\beta) \tag{$\mathcal{D}_{\epsilon,\eta}$}
  \end{equation}
  While $\Ll^K_{\epsilon,\eta}$ is only defined here for the purpose of the proof, in Section \ref{sec:fast} we will illustrate a case where we actually need to solve this problem. Solving the dual \eqref{eq:dualOT} is usually done by block coordinate-ascent, which is nothing more than Sinkhorn's algorithm for the primal \eqref{eq:primalOT}, but in the log-domain \cite{Peyre2019}. Solving \eqref{eq:dualOTbounded} can be done by simply adding a projection step on the constraints $\norm{f}_\infty, \norm{g}_\infty \leq \epsilon\log(\eta)$.

  Let $\hat K$ be a perturbed version of $K$.
  By a reasoning similar to the proof of Prop.~\ref{prop:stability_infty}, we have
  \begin{align}
    &\abs{\Ll_{\epsilon,\eta}^{\hat K}(\alpha,\beta) - \Ll_{\epsilon,\eta}^{K}(\alpha,\beta)}\notag \\
    &\quad \leq \sup_{f,g\in \Bb_{\epsilon \log \eta}^{\norm{\cdot}_\infty}} \abs{\Ll_\epsilon^K(f,g,\alpha,\beta) - \Ll_\epsilon^{\hat K}(f,g,\alpha,\beta)} \notag\\
    &\quad \leq \epsilon\eta^{2}\norm{\alpha}\norm{\beta}\norm{K-\hat K}\, .\label{eq:spectral_bound_on_K}
  \end{align}
  We obtain \eqref{eq:stability_spectral} by taking $K = e^{-C/\epsilon}$, $\hat K = e^{-\hat C/\epsilon}$ and $\eta = e^{(c_{\max}-c_{\min}/2)/\epsilon}$, such that by Lemma \ref{lem:dualbound} $\Ll^K_{\epsilon,\eta}(\alpha,\beta) = W_\epsilon^C(\alpha, \beta)$ and $\Ll^{\hat K}_{\epsilon,\eta}(\alpha,\beta) = W_\epsilon^{\hat C}(\alpha, \beta)$.

  Next we prove the stability of the transport plan. We know that optimal transport plans are insensitive to shifting the cost (see e.g. \eqref{eq:primalOT}, where replacing $C$ by $C+a$ does not change the minimization problem), so for the rest of the proof we assume without lost of generality that $C, \hat C$ are both shifted by $-c_{\min}$ such that $0 \leq C_{ij}, \hat C_{ij} \leq \bar c \eqdef c_{\max} - c_{\min}$. Strong convexity leads to the following result, proved in App.~\ref{app:stability_TP}.

  \begin{lemma}\label{lem:stability_TP}
    Under the conditions above, for any $f,g$ satisfying $\norm{f}_\infty, \norm{g}_\infty \leq \bar c$,
    \begin{align}
      &\sum\nolimits_{ij} K_{ij} \alpha_i \beta_j \abs{f_i + g_j - (f^K_i + g^K_j)}^2 \notag \\
      &\qquad\leq (\epsilon/2) e^{2\bar c/\epsilon}\pa{\Ll(f^K, g^K) - \Ll(f,g)}
    \end{align}
  \end{lemma}
  Using this lemma, by maximality of $\Ll^{\hat K}(f^{\hat K},g^{\hat K})$, using similar computation as in the proof of Prop.~\ref{prop:stability_infty}, once again:
  \begin{align*}
    &\sum\nolimits_{ij} K_{ij} \alpha_i \beta_j \abs{f^{\hat K}_i + g^{\hat K}_j - (f^K_i + g^K_j)}^2 \\
    &\leq (\epsilon/2) e^{2\bar c/\epsilon}\pa{\Ll^K(f^K, g^K) - \Ll^K(f^{\hat K},g^{\hat K})} \\
    &\leq \epsilon e^{2\bar c/\epsilon} \sup_{f,g} \abs{\Ll^K(f,g) - \Ll^{\hat K}(f,g)}\\
    &\leq \epsilon^2 e^{4\bar c/\epsilon} \norm{\alpha}\norm{\beta} \norm{e^{-C/\epsilon}-e^{-\hat C/\epsilon}}
  \end{align*}
  since the supremum is on $\norm{f}_\infty, \norm{g}_\infty \leq \bar c$.

  Now, for the optimal transport plans $P^C = (e^{f^K/\epsilon}\odot \alpha)e^{-C/\epsilon}(e^{g^K/\epsilon}\odot \beta)^\top$ and similarly $P^{\hat C}$, we bound
  \begin{align*}
    &\KL(P^C | P^{\hat C}) = \sum\nolimits_{ij} P^C_{ij} \log(P^C_{ij}/P^{\hat C}_{ij}) \\
    &= \sum\nolimits_{ij} e^{\frac{f^K_i + g^K_j - C_{ij}}{\epsilon}}\alpha_i \beta_j \\
    &\qquad \epsilon^{-1}\pa{f^K_i + g^K_j - f^{\hat K}_i - g^{\hat K}_j + C_{ij} - \hat C_{ij}}
  \end{align*}

  Applying twice Cauchy-Schwartz inequality,
  \begin{align*}
    &\KL(P^C | P^{\hat C}) \leq \epsilon^{-1}e^{\frac{2\bar c}{\epsilon}}\norm{\alpha}\norm{\beta}\norm{C-\hat C}_F\\
    &\quad + \epsilon^{-1} e^{2\bar c/\epsilon}\sqrt{\sum e^{-C_{ij}/\epsilon}\alpha_i \beta_j} \\
    &\qquad \sqrt{\sum e^{-C_{ij}/\epsilon}\alpha_i \beta_j \abs{f^{\hat K}_i + g^{\hat K}_j - f^K_i - g^K_j}^2} \\
    &\leq \epsilon^{-1} e^{\frac{2\bar c}{\epsilon}}\norm{\alpha}\norm{\beta}\norm{C-\hat C}_F\\
    &\quad + e^{4\bar c/\epsilon} \sqrt{\norm{\alpha}\norm{\beta} \norm{e^{-C/\epsilon}-e^{-\hat C/\epsilon}}}
  \end{align*}
  For the generic case where $C, \hat C$ are in $[c_{\min}, c_{\max}]$, we apply the bound above to their shifted version, which as we recall does not change the transport plan and conclude the proof.
\end{proof}

\section{Local kernels}\label{sec:local}

\begin{figure}[h]
  \centering
  \includegraphics[width=.15\textwidth]{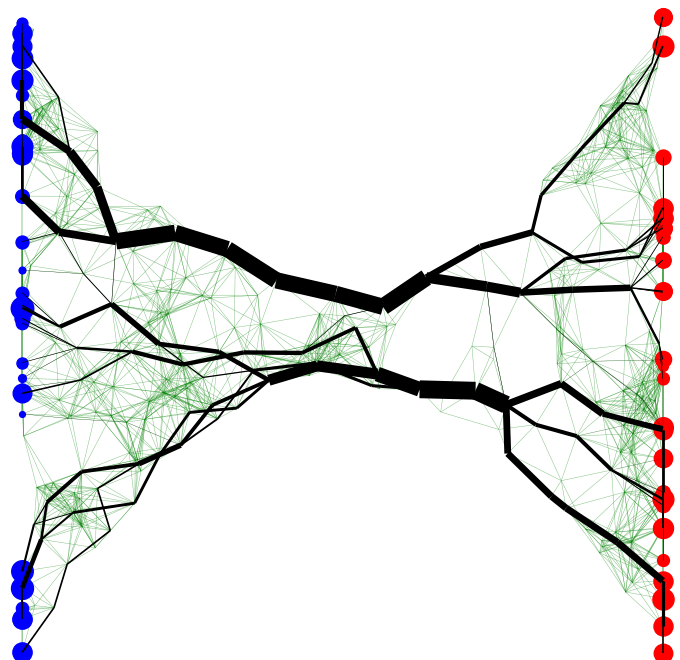}
  \includegraphics[width=.15\textwidth]{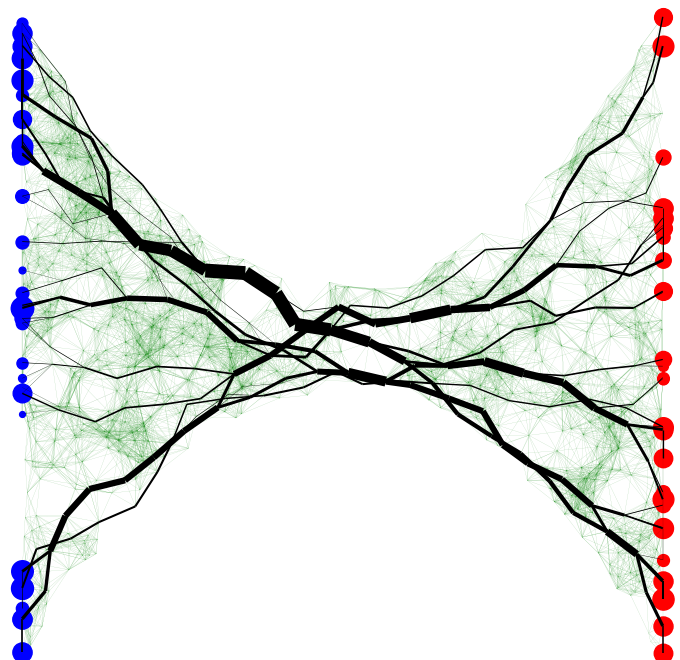}
  \includegraphics[width=.15\textwidth]{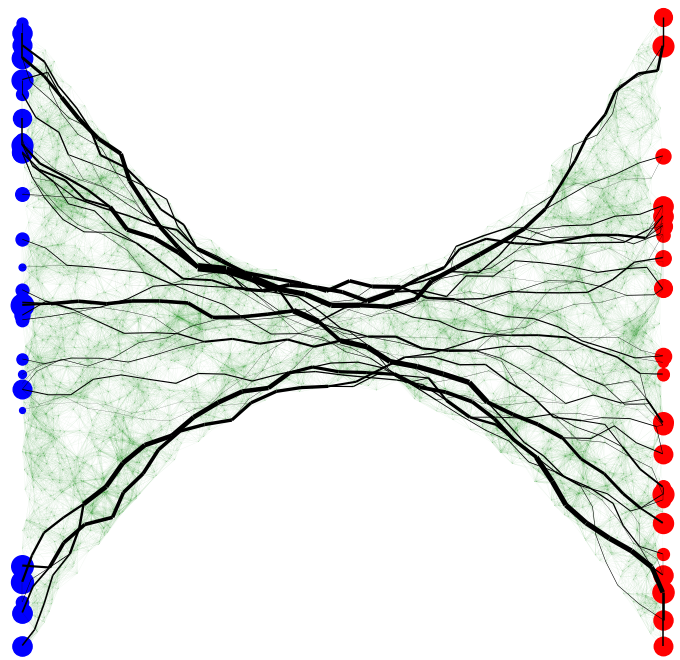}
  \caption{\small Optimal transport plan displayed along the shortest paths on $\epsilon$-graphs on a 2D compact domain, for increasing $N=300, 1000, 3000$.}
  \vspace{-10pt}
  \label{fig:local_TP}
\end{figure}

In this section, we consider ``local'' kernels with vanishing connectivity radius. Recall that the latent variables of the graph are divided into three groups: $\{x_1,\ldots x_n, y_1,\ldots, y_m, z_{n+m+1},\ldots, z_N\}$, and that the distributions $\alpha, \beta$ are respectively supported over the nodes with latent variables $x_i$ and $y_j$. We let $N\gg n+m$ go to $\infty$ and aim to use the auxiliary nodes $z_\ell$ to estimate some notion of cost between the target nodes $x_i,y_j$ (see Fig.~\ref{fig:local_TP}).
We look at classic ``$\epsilon$-graphs'', where two nodes are connected if their latent variables are closer than a threshold $h_N$:
\begin{equation*}
  w_N(z,z') = 1_{\norm{z-z'}\leq h_N}
\end{equation*}
Intuitively, in an graph with sufficient density of points and small radius $h_N$, the \emph{shortest path} between two points converges to a \emph{geodesic path}, that is, the limit continuous shortest path (see Fig.~\ref{fig:local_TP}). There are many settings in which this is true \cite{Bernstein2000, Alamgir2012, Hwang2016, Davis2019}, for various notions of geodesics.
More precisely, we assume that the latent variables belong to $\Mm$, a $k$-dimensional compact smooth submanifold $\Mm \subset \RR^d$ without boundary for simplicity, equipped with the Riemannian geometry induced by the Euclidean inner product in $\RR^d$. We denote by $d(z,z') \geq \norm{z-z'}$ the geodesic distance on $\Mm$ and $\norm{\cdot}$ the Euclidean norm in $\RR^d$. The diameter of $\Mm$ is $D_\Xx = \max_{z,z'\in\Mm} d(z,z')$.
The uniform measure on $\Mm$ is denoted by $\mu$. We assume that the auxiliary nodes $z_{n+m+1}, \ldots, z_N$ are distributed i.i.d. according to some measure $\nu$ on $\Mm$, which we assume to have a density $p$ wrt $\mu$, lower bounded by $p(z)\geq c_z$. On the contrary, we do not make any assumptions on the points $x_i,y_j$. We let $N\to \infty$ and $h_N\to 0$, and potentially $n+m = o(N)$ grow as well, such that
\begin{equation}\label{eq:rate_epsilon}
  \frac{N h_N^k}{\log(nm/h_N)} \to \infty
\end{equation}

We consider an OT cost that is a function of the geodesic distance $c(x,y)=f(d(x,y))$, where $f:[0,D_\Xx] \to [c_{\min}, c_{\max}]$ is $c_f$-Lipschitz. It is known that various estimators converge to $d(x,y)$, for instance the weighted shortest path $\min \sum_\ell \norm{z_{i_{\ell+1}} - z_{i_\ell}}$, where the minimization is over all paths $z_{i_0} = x, z_{i_1}, \ldots, z_{i_L} = y$, is the basis of the classic ISOMAP algorithm \cite{Bernstein2000}, and converges to $d(x,y)$. Various other procedures leads to different geodesic metrics \cite{Sajama2005, Alamgir2012, Davis2019}. Such estimators can directly lead to bounds on the largest deviation of the cost matrix $\norm{C-\hat C}_\infty$, which combined with Prop.~\ref{prop:stability_infty} is sufficient to obtain stability bounds on the OT distance for any $\epsilon$. Stability of the transport plan is obtained with Thm.~\ref{thm:stability} and valid only for $\epsilon>0$.

In the spirit of this paper however, here we consider that the latent variables are \emph{unknown} as well as their pairwise distance, but only the radius $h_N$ is known. The ISOMAP estimator therefore cannot be computed, and instead the shortest path estimator is taken as 
\begin{equation}\label{eq:shortest_path}
  \hat d_{ij} = h_N^{-1} \text{SP}(i,j)\, ,
\end{equation}
where $\text{SP}(i,j)$ is the length of the shortest path of unweighted edges (that is, simply the number of edges) in the graph, between the vertices corresponding to $x_i$ and $y_j$. We take $\hat C = [f(\hat d_{ij})]_{ij}$, and shall prove that $\hat d_{ij}$ converges to $d(x_i,y_j)$ which, to the best of our knowledge, is a novel result for geodesic convergence with \emph{unweighted} edges.

We first recall a few facts. It turns out that, for smooth and compact manifolds, $d(\cdot,\cdot)$ and $\norm{\cdot}$ are equivalent up to order three \cite{Belkin2008}: we let $h_\Mm, c_\Mm$ such that for all $\norm{x-y}\leq h_\Mm$,
\begin{equation}
  \abs{\norm{x-y} - d(x,y)} \leq c_\Mm \norm{x-y}^3 \label{eq:geodesic}
\end{equation}
Moreover, for $k$-dimensional manifolds a ball $\Bb_h(x) = \{y\in \Mm;d(x,y)\leq h\}$ has measure $\mu(\Bb_h(x)) \geq c_\Bb h^k$ for some constant $c_\Bb$. We have the following result.
\begin{theorem}\label{thm:geodesic}
  For $N$ large enough, with probability at least $1-\rho$, we have: for all $i,j$,
  \begin{align*}
    -c_\Mm h_N^2 (1+ R_N) \leq \frac{\hat d_{ij}}{d(x_i,y_j)} -1 \leq R_N
  \end{align*}
  where
  \begin{equation*}
    R_N \propto c_\Mm h_N + \pa{\frac{\log\frac{D_\Xx n m}{h_N \rho}}{c_z c_\Bb N h_N^k}}^{1/k} \xrightarrow[N\to \infty]{} 0
  \end{equation*}
  In particular, $\sup_{ij}\abs{\hat d_{ij} - d(x_i,y_j)} \leq D_\Xx R_N$.
\end{theorem}

Note that the estimator does not depend on the measure $\nu$. In fact, the proof shows that, since $\nu$ has a lower-bounded density, at any position on $\Mm$ there is always a node $z_i$ at distance ``about'' $h_N$ in any direction, including that of the geodesic path of interest. On the contrary, it is known for instance that this result does not hold if $\epsilon$-graphs are replaced by $k$-Nearest Neighbor graphs \cite{Alamgir2012}, which are strongly sensitive to the density $\nu$.
Combined with Thm.~\ref{thm:geodesic} and the Lipschitz property of $f$, Prop.~\ref{prop:stability_infty} and Thm.~\ref{thm:stability} yield the following stability bounds.

\begin{corollary}\label{cor:geodesic_ot}
  For $N$ large enough, with probability at least $1-\rho$, we have for all $\epsilon \geq 0$:
  \begin{equation*}
    \abs{\Ww_\epsilon^C(\alpha, \beta) - \Ww_\epsilon^{\hat C}(\alpha, \beta)} \leq c_f D_\Xx R_N
  \end{equation*}
  And for all $\epsilon>0$ and all distributions satisfying $\norm{\alpha}_\infty \leq c_\alpha/n$ and $\norm{\beta}_\infty \leq c_\beta/m$,
  \begin{equation*}
    \KL(P^C|P^{\hat C}) \lesssim \epsilon^{-\frac12}e^{4(c_{\max} - c_{\min})/\epsilon}\sqrt{c_\alpha c_\beta c_f D_\Xx R_N}
  \end{equation*}
\end{corollary}

\paragraph{Numerical illustration.} In Fig.~\ref{fig:local_TP} we give a simple numerical illustration of the optimal transport plan on a 2D domain (even though it is technically not a smooth manifold without boundary). In Fig.~\ref{fig:local}, we illustrate the convergence bounds of Cor.~\ref{cor:geodesic_ot} on the 3D sphere, where the true geodesics are known. It can be seen that, unlike the theory predicted, the convergence of $\KL(P^C|P^{\hat C})$ does not seem to be slower than that of the OT distance itself. Note that in both cases we use a non-uniform measure $\nu$. 

\begin{figure}[h]
  \centering
  \includegraphics[width=.15\textwidth]{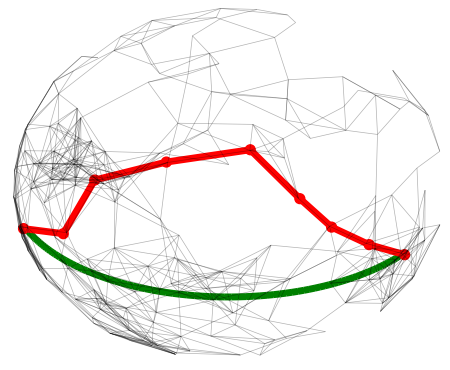}
  \includegraphics[width=.15\textwidth]{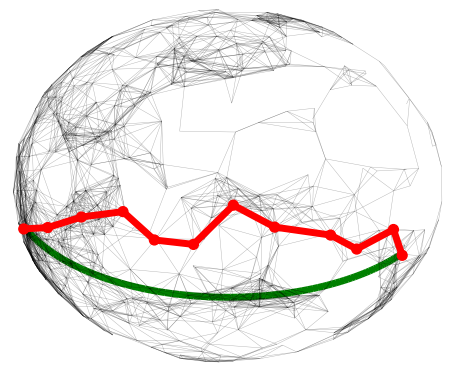}
  \includegraphics[width=.15\textwidth]{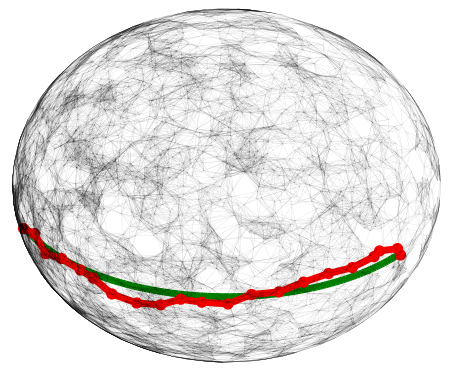} \\
  \includegraphics[height=2.9cm]{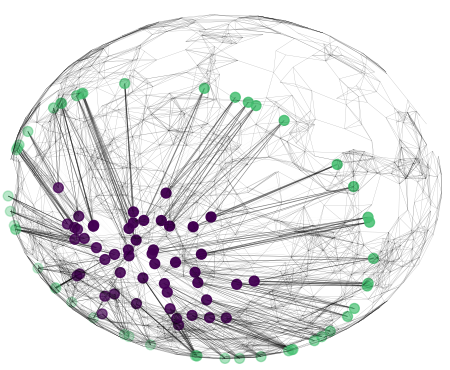}
  \includegraphics[height=2.9cm]{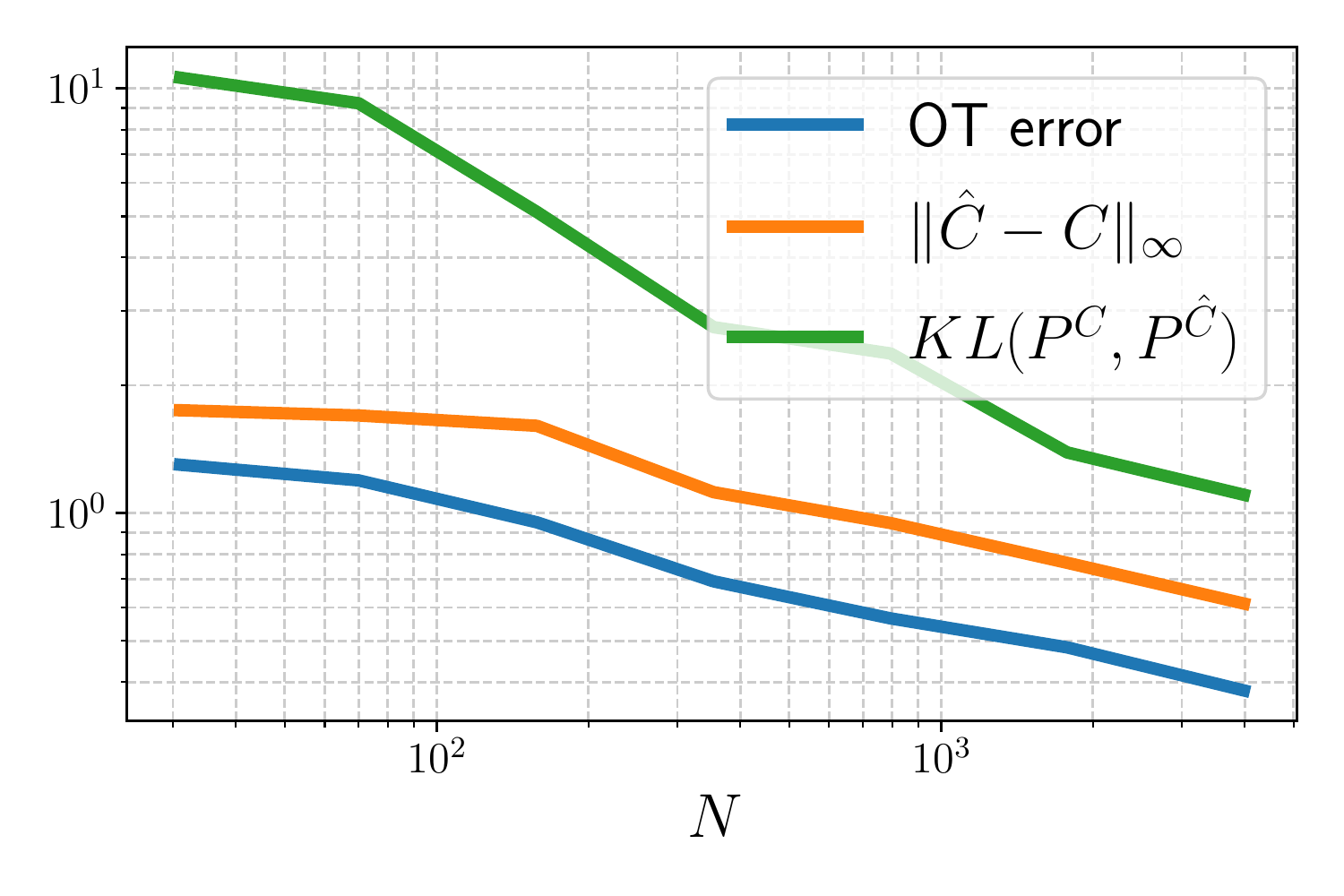}
  \caption{\small Shortest paths on $\epsilon$-graphs on the 3D sphere. \textbf{Top:} illustration of convergence of the shortest path in the graph (red) to the true shortest path (green). \textbf{Bottom left:} example of transport plan on the sphere using geodesic cost. \textbf{Bottom right:} convergence of the norm $\norm{\hat C - C}_\infty$, the normalized error $\abs{1 - \frac{\Ww_\epsilon^{\hat C}}{\Ww_\epsilon^{C}}}$, and the KL divergence $\KL(P^{C}|P^{\hat C})$.}
  \vspace{-10pt}
  \label{fig:local}
\end{figure}

\section{Non-local kernels}\label{sec:nonlocal}

In this section, we look at RGs with \textbf{non-local kernels} $w_N = \rho_N w$, for a fixed kernel $w$ and sparsity factor $\rho_N \gtrsim \log N/N$. This regime is usually referred to as \emph{relatively sparse} \cite{Araya2019}, that is, the expected number of edges in the random graph evolves as $N\log N$. We denote by $W\in[0,1]^{N \times N}$ the matrix $W = \EE(A/\rho_N)$ containing the true values of the kernel $w$ between pairs of points.
In these settings, it is known that the adjacency matrix of the graph will somewhat concentrate around its expectation \cite{Lei2015}. Hence, if the OT cost is related to the kernel $w$, the adjacency matrix may directly be an estimation of the cost matrix, unlike the previous section, where the shortest paths in the graph were the quantities of interest. Of course, the \emph{individual} elements of the adjacency matrix will \emph{not} concentrate, and the bound in Prop.~\ref{prop:stability_infty} will not be sufficient. Instead, we shall use the bounds in Thm.~\ref{thm:stability}, valid only for non-zero entropic regularization $\epsilon>0$.

Here we will see that only the edges between the target nodes $x_i,y_j$ will be used in our estimators. Hence we assume that $N=n+m$ and that the latent variables are simply $\{x_1,\ldots, x_n, y_1,\ldots, y_m\}$ (i.e. there is no ``auxiliary'' nodes $z_{n+m+i}$ or they are ignored). We take $n \sim m$ and let $N=n+m \to \infty$.  For simplicity, we assume that $\rho_N$ is known (or estimated).
We present two strategies: a generic estimator that works for any positive semidefinite kernel $w$, and a particular case for specific kernel and fixed $\epsilon$, where a more direct estimator leads to faster rates of convergence.

\subsection{USVT estimator}\label{sec:usvt}

In this section we assume that $w(z,z')$ is a \emph{positive semi-definite kernel} \cite{Berlinet2004} satisfying $0 \leq w_{\min} \leq w(z,z') \leq w_{\max} \leq 1$. We take a cost function of the form:
\begin{equation}\label{eq:usvt_cost}
  c(z,z') = f(w(z,z'))
\end{equation}
For some $c_f$-Lipschitz function $f:[w_{\min},w_{\max}] \to [c_{\min},c_{\max}]$. Hence $C = f(W_{1:n, n+1:N})$, and we would like to estimate $\hat W$ and take $\hat C = f(\hat W_{1:n, n+1:N})$. In practice, the kernel $w$ is of course unknown, however it is reasonable to assume that it decreases when $z,z'$ get further away from each other, hence $f$ is generally chosen as a decreasing function. For instance, we use $f(x)=1-x$ and a Gaussian kernel $w(z,z') = e^{-\frac{\norm{z-z'}^2}{2\sigma^2}}$ in our experiments (Fig.~\ref{fig:nonlocal_single}, \ref{fig:nonlocal}).

It is known that $A/\rho_N$ directly concentrates around $W$ \emph{in operator norm} \cite{Lei2015} but not in Frobenius norm, as needed by Thm.~\ref{thm:stability} and Prop.~\ref{prop:spec2F}. However, convergence can be restored using the so-called USVT estimator \cite{Chatterjee2015}. If $A$ is diagonalized as $A = \sum_i \sigma_i a_i a_i^\top$ for an orthonormal basis $\{a_i\}$, the USVT estimator is a low-rank approximation defined as
\begin{equation}\label{eq:usvt}
  \hat W_\gamma \eqdef \textup{HT}_{[w_{\min},w_{\max}]}\Big(\rho_N^{-1}\sum_{\sigma_i \geq \gamma\sqrt{\rho_N N}} \sigma_i a_i a_i^\top\Big)
\end{equation}
where $\gamma>0$ is some constant and $\textup{HT}_{[w_{\min},w_{\max}]}$ is a hard thresholding function that projects each entry onto $[w_{\min},w_{\max}]$. The following lemma is adapted from \cite{Chatterjee2015} combined with a result in \cite{Lei2015}.
\begin{theorem}\label{thm:usvt}
  For any $r>0$, there are two constants $\gamma_r, c_r$ such that the following holds. With probability at least $1-N^{-r}$, we have
  \begin{equation}\label{eq:usvt_bound}
    \frac{1}{N} \norm{\hat W_{\gamma_r} - W}_F \leq \frac{c_r}{(\rho_N N)^{1/4}}\, .
  \end{equation}
\end{theorem}

We can now define $\hat C$ as the appropriate rectangular part $\hat C_{\gamma_r} \eqdef f\pa{(\hat W_{\gamma_r})_{1:n, n+1:N}}$. Thm.~\ref{thm:usvt} combined with Thms.~\ref{thm:stability} leads to the following result.

\begin{corollary}\label{cor:usvt}
  For any $r, \epsilon>0$, there are two constants $\gamma_r, c_r$ such that the following holds. With probability at least $1-N^{-r}$: for all distributions satisfying $\norm{\alpha}_\infty \leq c_\alpha/n$ and $\norm{\beta}_\infty \leq c_\beta/m$,
  \begin{equation*}
    \abs{W_\epsilon^C(\alpha, \beta) - W_\epsilon^{\hat C_{\gamma_r}}(\alpha, \beta)} \lesssim \frac{c_r c_f c_\alpha c_\beta e^{2(c_{\max}-c_{\min})/\epsilon}}{(\rho_N N)^{1/4}} \\
  \end{equation*}
  and
  \begin{equation*}
    \KL(P^C|P^{\hat C}) \lesssim \epsilon^{-\frac12}e^{4(c_{\max}-c_{\min})/\epsilon}\frac{\sqrt{c_r c_f c_\alpha c_\beta}}{(\rho_N N)^{1/8}}
  \end{equation*}
\end{corollary}

As remarked in the original USVT paper \cite{Chatterjee2015}, despite its good theoretical properties the estimator $\hat W_\gamma$ may be difficult to use in practice, as the constant $\gamma$ can be hard to adjust. It is however a good inspiration to combine with methods that learn the cost for robust OT \cite{Carlier2020, Dhouib2020}, which we leave for future work.

\begin{proof}[Sketch of proof of Thm.~\ref{thm:usvt}.] The proof is based on the following concentration results on symmetric matrices with Bernoulli entries, due to Lei and Rinaldo \cite{Lei2015}.
\begin{theorem}[\cite{Lei2015}]\label{thm:Lei}
  For any $r>0$, there is a constant $c_r$ such that the following holds. With probability at least $1-n^{-r}$ we have:
  \begin{equation}
    \norm{A - \rho_N W} \leq c_r \sqrt{\rho_N N}
  \end{equation}
\end{theorem}
Then, using a similar strategy to \cite{Chatterjee2015}, we show in the appendix that
\begin{equation*}
  \norm{\hat W_\gamma - W}_F \lesssim \sqrt{\norm{W}_\star\norm{A/\rho_N - W}}
\end{equation*}
where $\norm{\cdot}_\star$ is the nuclear norm. Since $w$ is a p.s.d. kernel, $W$ is p.s.d., and $\norm{W}_\star = Tr(W) \leq n$. Theorem \ref{thm:Lei} concludes the proof.
\end{proof}

\begin{figure}[ht]
  \centering
  \includegraphics[width=.15\textwidth]{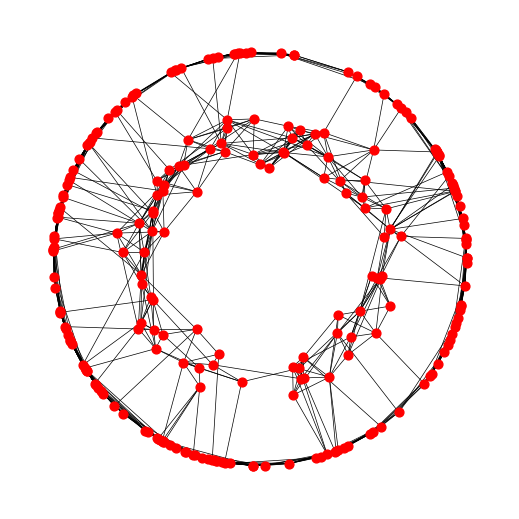}
  \includegraphics[width=.15\textwidth]{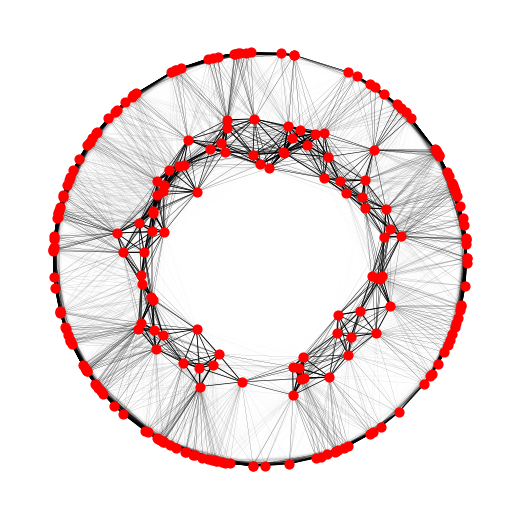}
  \includegraphics[width=.15\textwidth]{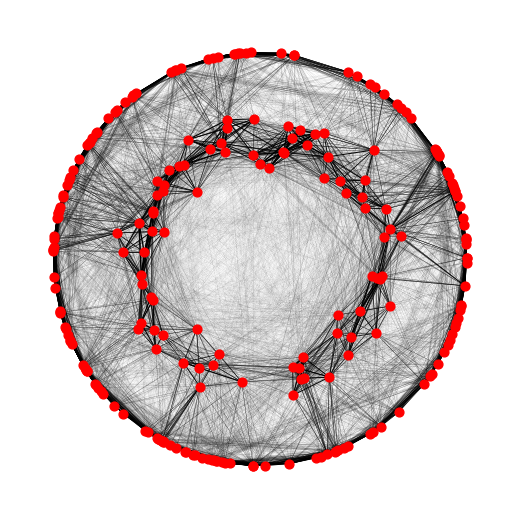}\\
  \includegraphics[width=.22\textwidth]{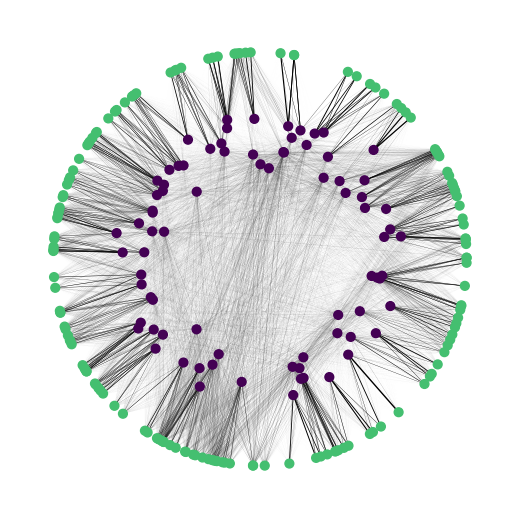}
  \includegraphics[width=.22\textwidth]{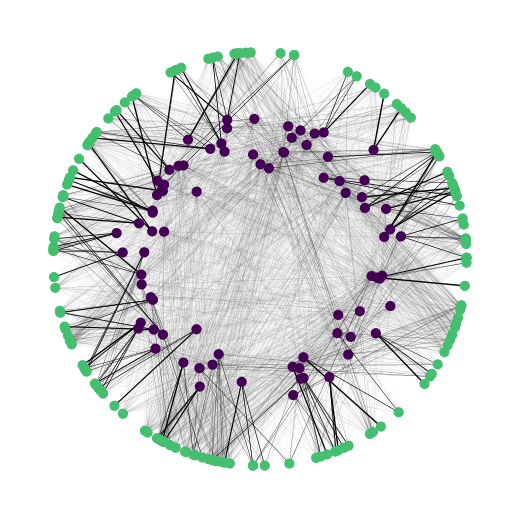}
  \caption{\small \textbf{Top,} from left to right: random graph, true edge weights, and edge weights estimated with the USVT estimator. \textbf{Bottom:} comparison between the true transport plan $P$ obtained when computing $\Ww_{\epsilon}^C$ and the one obtained when computing $\Ww_{\epsilon}^{\hat C_{\gamma_r}}$.}
  \label{fig:nonlocal_single}
\end{figure}

\subsection{Fast rate with Gaussian-like kernel}\label{sec:fast}

The convergence result of Thm. \ref{thm:usvt} uses the version of the stability bound in Prop.~\ref{prop:spec2F} involving the Frobenius norm between $\hat C$ and $C$. In this section, we illustrate a specific case exploiting the operator norm in \eqref{eq:stability_spectral} which leads to faster rates of convergence.
We still consider non-local kernel $w_N = \rho_N w$, but here \emph{specifically} with a kernel of the form:
\begin{equation}\label{eq:gaussian_kernel}
  w(x,y) = e^{-\frac{\norm{x-y}^p}{\sigma}}
\end{equation}
such as, for instance, the Gaussian kernel when $p=2$. We assume that $\sigma$ is known. One then notices that the matrix $W$ containing the $w(x_i,y_j)$ \emph{directly} looks like the matrix $K = e^{-C/\epsilon}$, for the very specific choices:
\begin{equation}\label{eq:fast_choices}
  c(x,y) = \norm{x-y}^p,\quad \epsilon = \sigma
\end{equation}
If taking $c$ as a power of the Euclidean distance is a classic choice (leading to the so-called $p$-Wasserstein distance), here one notes that the choice of $\epsilon$ is imposed, which is a major shortcoming compared to the previous ``universal'' strategy.

Nevertheless, in this case we can directly define an estimator of the matrix $K$ as the normalized rectangular part of the adjacency matrix:
\begin{equation}\label{eq:fast_estimator}
  \hat K = \rho_N^{-1} A_{1:n, n+1:N}
\end{equation}
Remark that this estimator is extremely simple, in particular it only uses the edges between the $x_i$ and the $y_j$ as if the graph were bipartite, unlike the USVT estimator which uses the whole adjacency matrix.

We could directly plug this estimator into the dual problem \eqref{eq:dualOT}. Unfortunately, $\hat K$ is not bounded away from $0$, so one cannot apply Lemma \ref{lem:dualbound} to bound the dual potential and carry on with the proof of the stability bounds like in Theorem \ref{thm:stability}. Instead, one has to directly enforce box constraints, and we will instead solve \eqref{eq:dualOTbounded} to obtain some $\Ll^{\hat K}_{\epsilon,\eta}(\alpha,\beta)$, for some $\eta$. As mentioned earlier, this can be handled with a block-coordinate ascent with an additional projection step. It leads to the following result proved in App.~\ref{app:fast}, whose convergence rate is twice as fast as the bound of Thm.~\ref{thm:usvt}. Note however that in this case we do not have convergence in Frobenius norm.

\begin{theorem}\label{thm:fast}
  Define $c_{\min}, c_{\max}$ such that $0\leq c_{\min} \leq C_{ij} \leq c_{\max}$ and pick $\eta \geq e^{\frac{c_{\max}-c_{\min}/2}{2\sigma^2}}$. For any $r>0$, there is a constant $c_r$ such that the following holds. With probability at least $1-N^{-r}$: for all distributions satisfying $\norm{\alpha}_\infty \leq c_\alpha/n$ and $\norm{\beta}_\infty \leq c_\beta/m$,
  \begin{equation}\label{eq:fast_ot_bound}
      \abs{\Ww^{C}_{\sigma}(\alpha,\beta) - \Ll^{\hat K}_{\sigma,\eta}(\alpha,\beta)} \lesssim \frac{c_r c_\alpha c_\beta \sigma \eta^2}{\sqrt{\rho_N N}}
  \end{equation}
\end{theorem}

\begin{figure}
  \centering
  \includegraphics[width=.22\textwidth]{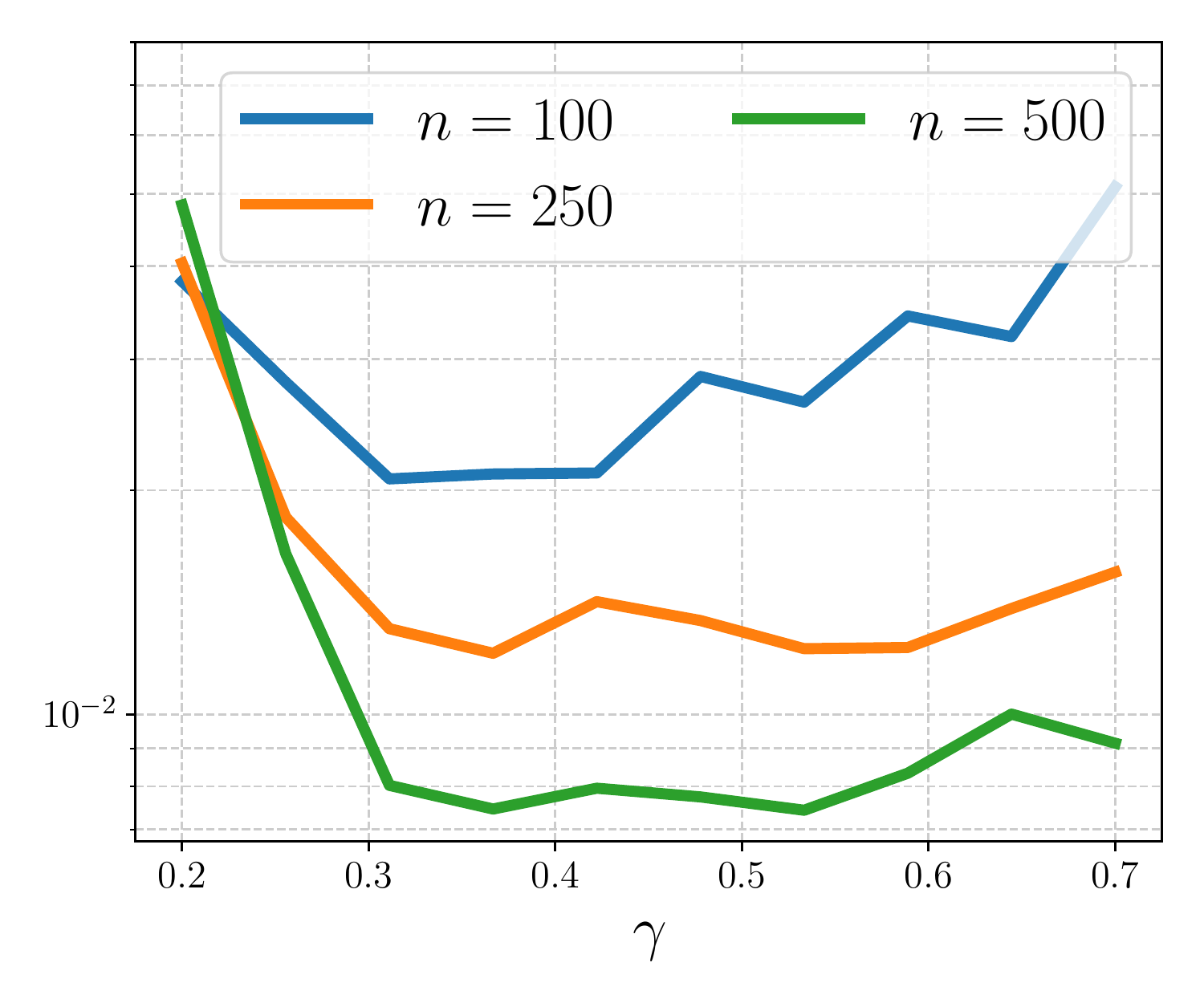}
  \includegraphics[width=.22\textwidth]{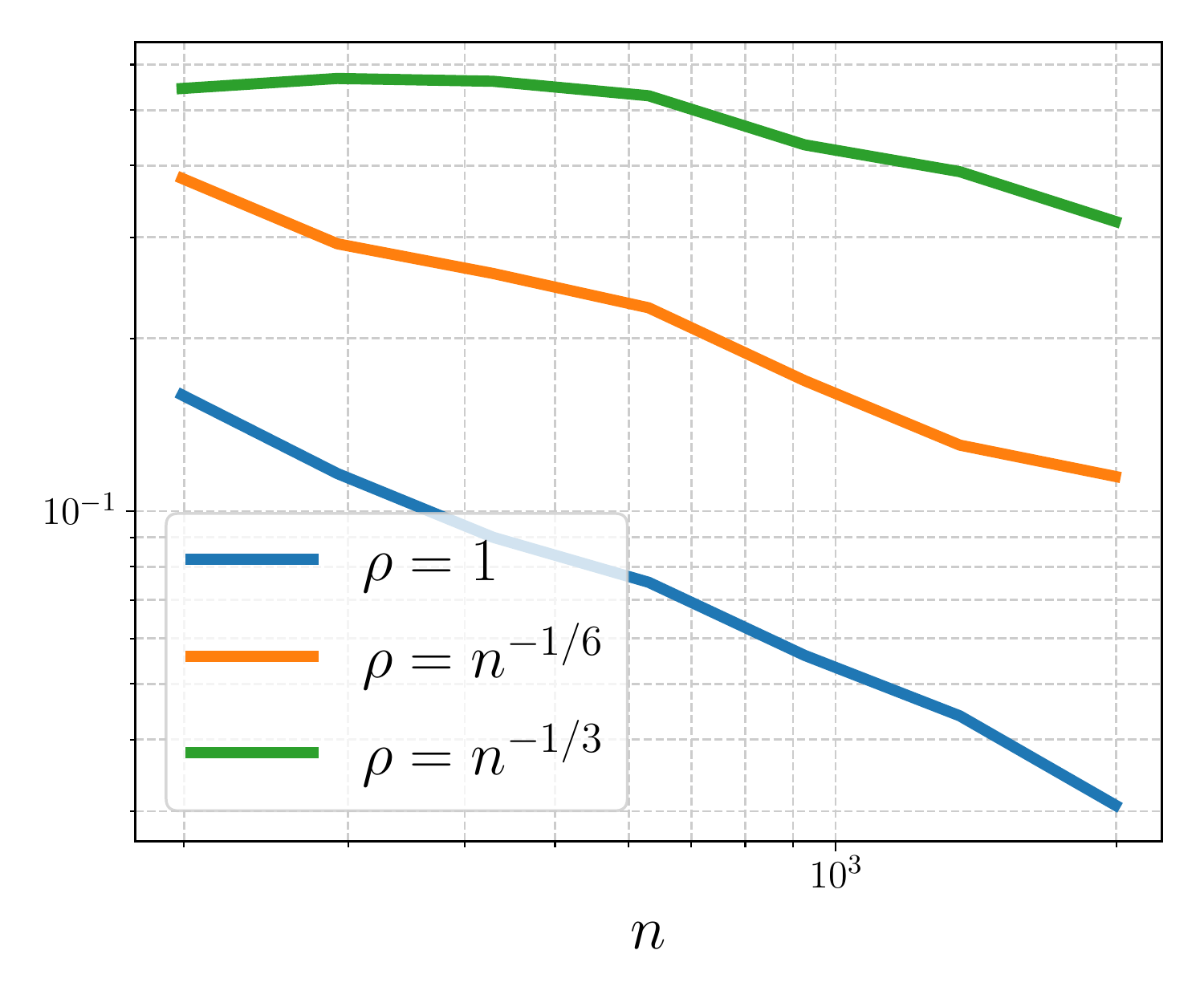} \\
  \includegraphics[width=.22\textwidth]{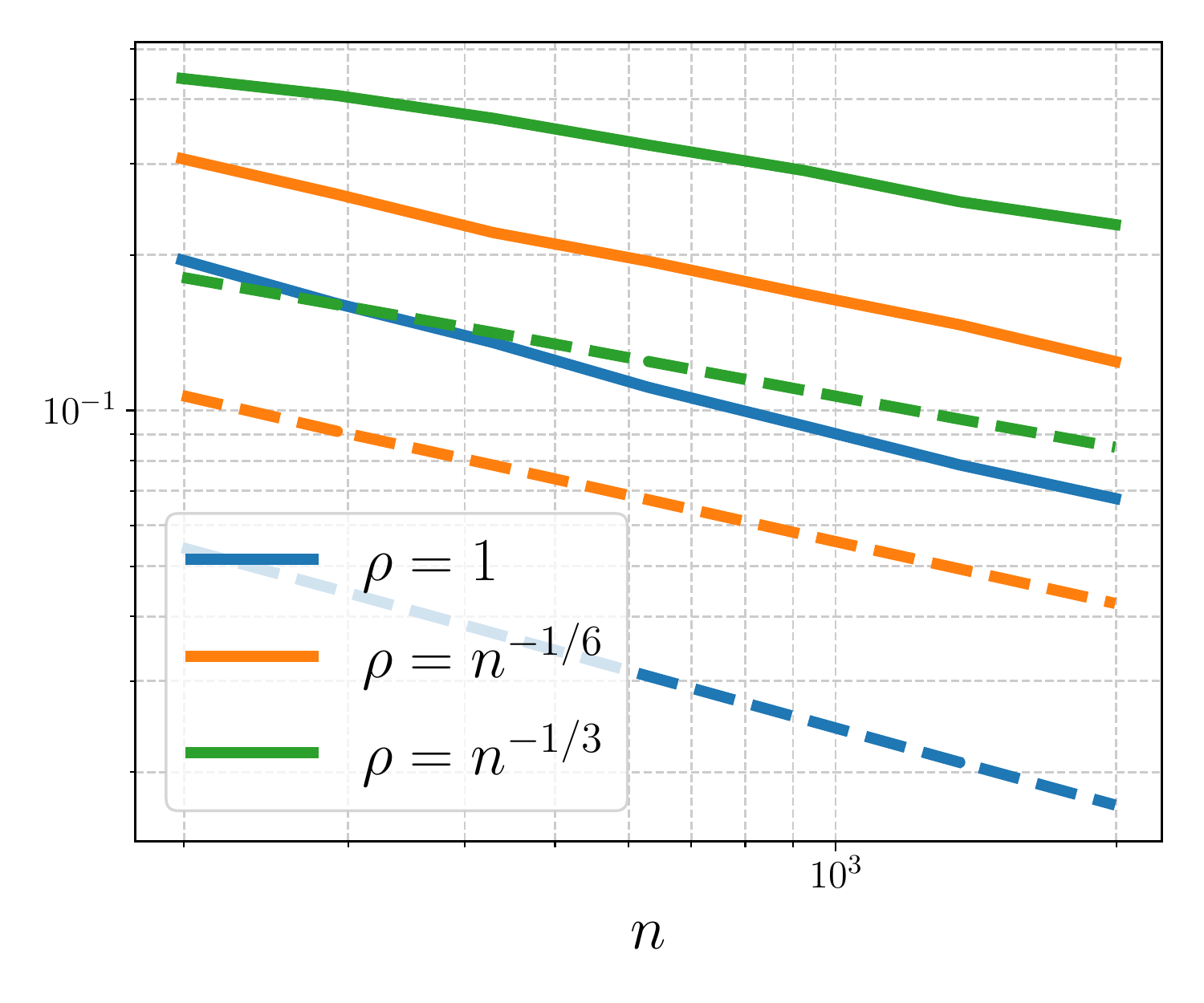}
  \includegraphics[width=.22\textwidth]{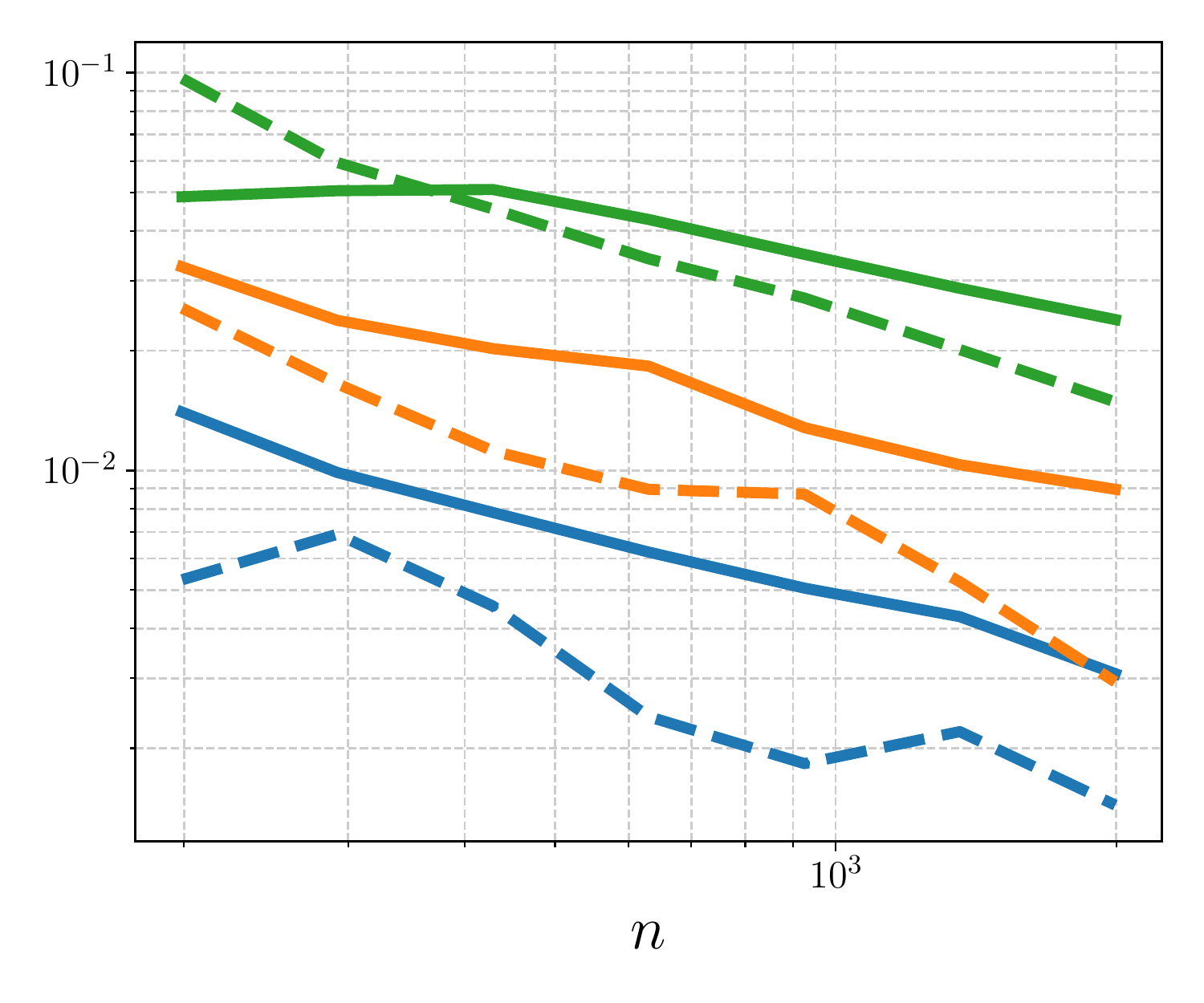}
  \caption{\small \textbf{Top left:} Stability of the normalized error $\abs{1 - \frac{\Ww_\epsilon^{\hat C_{\gamma_r}}}{\Ww_\epsilon^{C}}}$ w.r.t. $\gamma$. \textbf{Top right:} Convergence of the transport plan $\KL(P^C|P^{\hat C})$, for different sparsity levels $\rho_N$. \textbf{Bottom left:} norm $\frac{1}{N}\norm{\hat W_{\gamma_r}-W}_F$ (full line) or $\frac{1}{N}\norm{\hat K - K}$ (dotted line). \textbf{Bottom right:} normalized error $\abs{1 - \frac{\Ww_\epsilon^{\hat C_{\gamma_r}}}{\Ww_\epsilon^{C}}}$ for the USVT (full line) or $\abs{1 - \frac{\Ll_{\epsilon,\eta}^{\hat K}}{\Ww_\epsilon^{C}}}$ for the fast (dotted line) estimators.}
  \vspace{-10pt}
  \label{fig:nonlocal}
\end{figure}

\subsection{Numerical illustration}

We conclude this section by simple illustrative experiments. We generate random graphs with non-local Gaussian kernel with the nodes divided in two groups, as shown in Fig.~\ref{fig:nonlocal_single}, with $m=2n$ and $n \to \infty$. For the USVT estimator, we use the cost \eqref{eq:usvt_cost} with $f(w)=1-w$. Recall that for the ``fast'' estimator \eqref{eq:fast_estimator}, the cost and regularization parameter $\epsilon$ are fixed by \eqref{eq:fast_choices}.
In Fig.~\ref{fig:nonlocal} (top left), we examine the stability of the USVT estimator with respect to the parameter $\gamma_r$. Surprisingly, the estimation of $\Ww_\epsilon^C$ does seem quite robust to the choice of $\gamma$ in our example, particularly when $n$ is high. Future investigation will seek to quantify this phenomenon and introduce an estimation procedure for $\gamma$.
In Fig.~\ref{fig:nonlocal} (bottom), we compare convergence rates of the two estimators, indeed observing that the direct estimator \eqref{eq:fast_estimator} is faster than the USVT estimator, while being less flexible.

\section{Conclusion and outlooks}

In this paper, we have shown that estimation of Wasserstein distances between nodes in latent position random graphs is theoretically possible, despite the fact that the latent positions are not known in general. The proofs are modular and indicate which conditions any estimator must satisfy for this to be true. We gave three distinct examples related to classical random graphs.
Our theoretical work hints at many potential outlooks. We have generally assumed, for simplicity, that several parameters such as $\rho_N, h_N$ were known. Depending on the context, they can be estimated. For instance, another way of dealing with unknown sparsity is to use the normalized Laplacian, which automatically removes the dependency on $\rho_N$ \cite{Keriven2020}, but leads to a different kernel. Future work will also examine more practical applications of OT in graphs and compare it to other methods, for clustering or to compute distance-based node embeddings \cite{Rossi2020}. As another example, OT \emph{barycenters} \cite{Agueh2011} on $\epsilon$-graphs might be a fast and consistent way to compute geodesic barycenters on manifolds \cite{Peyre2010}. Finally, at the graph level, infinite-node limits of the Gromov-Wasserstein distance \cite{Memoli2011, Memoli2014} are still to be properly studied.

\small

\bibliographystyle{myabbrvnat}
\bibliography{library}

\normalsize
%\clearpage
\appendix
\section{Additional proofs}

\subsection{Proof of Lemma \ref{lem:dualbound}}

For simplicity, we denote $\Ll(f,g) = \Ll^K(f,g,\alpha,\beta)$.
Let us first consider the case where $\delta_{\max} = 1$.
Take $(f^*, g^*)$ solution of $\max_{f,g} \Ll(f,g)$. Since $\alpha^\top K \beta \leq 1$, we have $\Ll(f^*,g^*) \geq \Ll(0,0) \geq 0$. Moreover, by first-order conditions we have $f^* = -\epsilon \log(K(e^{g^*/\epsilon} \odot \beta))$, and therefore $(e^{f^*/\epsilon} \odot \alpha)^\top K (e^{g^*/\epsilon} \odot \beta)=1$ and $\Ll(f^*,g^*) = \alpha^\top f^* + \beta^\top g^*$. Since taking $(f^* + c 1_n, g^*-c 1_m)$ for any $c \in \RR$ does not change the cost function, without lost of generality we assume that $\alpha^\top f^* = \beta^\top g^* = \frac12 \Ll(f^*,g^*) \geq 0$.
Using the above identity, Jensen's inequality and the fact that $\beta^\top g^* \geq 0$, we have for all $i$
\begin{align*}
  f^*_i &= -\epsilon \log\sum_j e^{g^*_j/\epsilon}\beta_j K_{ij} \\
  &\leq -\epsilon \sum_j \beta_j \log\pa{e^{g^*_j/\epsilon} K_{ij}} \\
  &=-\beta^\top g^* - \epsilon \sum_j \beta_j \log(K_{ij}) \leq \epsilon \log(1/\delta_{\min})
\end{align*}
Similarly, $g^*_j \leq \epsilon \log(1/\delta_{\min})$ by the same reasoning. Then, since $K_{ij}\leq 1$ we have $\sum_j e^{g^*_j/\epsilon}\beta_j K_{ij} \leq 1/\delta_{\min}$, and by the same identity
\begin{align*}
  f^*_i &= -\epsilon \log\sum_j e^{g^*_j/\epsilon}\beta_j K_{ij} \geq -\epsilon \log(1/\delta_{\min})
\end{align*}
and similarly for $g^*$, hence $\norm{f^*}_\infty, \norm{g^*}_\infty \leq \epsilon \log(1/\delta_{\min})$.
In the general case, we define $\tilde K = K/\delta_{\max}$. Considering $(\tilde f, \tilde g)$ solution of $\max_{f,g} \Ll^{\tilde K}(f,g)$, from what precedes we have $\norm{\tilde f}_\infty, \norm{\tilde g}_\infty \leq \epsilon \log(\delta_{\max}/\delta_{\min})$, and by the first order conditions the couple $(\tilde f + \epsilon \log(\delta_{\max})/2, \tilde g + \epsilon \log(\delta_{\max})/2)$ is solution of $\max_{f,g} \Ll^K(f,g)$ with the original $K$. Using $\abs{\epsilon \log(\delta_{\max})/2} = \epsilon \log(1/\sqrt{\delta_{\max}})$, we conclude.

\subsection{Proof of Lemma \ref{lem:stability_TP}}\label{app:stability_TP}

We write $\Ll(f,g) \eqdef \Ll^K(f,g) \eqdef \Ll_\epsilon^K(\alpha, \beta, f, g)$ for simplicity, and $f^C, g^C$ maximizing potentials which satisfy $\norm{f^C}_\infty, \norm{g^C}_\infty \leq \bar c$ by Lemma \ref{lem:dualbound}. Note that first order conditions state that $\nabla \Ll(f^C, g^C) = 0$.

Recall that the function $\phi: x \to e^{x/\epsilon}$ is $e^{a/\epsilon}/\epsilon^2$-strongly convex on the interval $[a,b]$ (by lower-bounding its second derivative) and therefore we have, for all $x,x' \in [a,b]$ and $t \in [0,1]$: $\phi(tx + (1-t)x') \leq t\phi(x) + (1-t)\phi(x') - \frac{e^{a/\epsilon}}{2\epsilon^2} t(1-t) |x-x'|^2$. Hence, for any $f,g$ satisfying $\norm{f}_\infty, \norm{g}_\infty \leq \bar c$ and $t$, we have
\begin{align*}
  &\Ll(tf + (1-t)f^K, tg + (1-t)g^K)\\
  &= t(\alpha^\top f + \beta^\top g) + (1-t)(\alpha^\top f^C + \beta^\top g^C) \\
  &\quad - \epsilon \sum_{ij}K_{ij} \alpha_i \beta_j e^{\frac{t(f_i + g_j)+(1-t)(f^C_i + g^C_j)}{\epsilon}} + \epsilon \\
  &\geq t \Ll(f,g) + (1-t)\Ll(f^C, g^C) \\
  &\quad+ \epsilon\sum_{ij} K_{ij} \alpha_i \beta_j \frac{e^{-2\bar c/\epsilon}t(1-t)}{2\epsilon^2} \abs{f_i + g_j - (f^C_i + g^C_j)}^2
\end{align*}

We divide by $t$ and take the limit $t \to 0$: using the fact that $\frac{\Ll(tf + (1-t)f^C, tg + (1-t)g^C) - \Ll(f^C, g^C)}{t} \to \nabla \Ll(f^C, g^C)^\top [f-f^C, g-g^C] = 0$, we conclude the proof.
  
\subsection{Proof of Theorem \ref{thm:usvt}}

By Theorem \ref{thm:Lei}, with probability at least $1-n^{-r}$, we have
\begin{equation}
  \norm{A - \rho_n W} \leq c_r \sqrt{\rho_n n}
\end{equation}
Assume that this is satisfied.
Decompose $W = \sum_i \tau_i w_i w_i^\top$. Denote by $S \subset \{1,\ldots, n\}$ the indices such that $\sigma_i \geq \gamma\sqrt{\rho_n n}$. Define
\begin{align*}
  G = \sum_{i \in S} \tau_i w_i w_i^\top, \qquad \hat A = \sum_{i \in S} \sigma_i a_i a_i^\top
\end{align*}
such that $\hat W = \textup{HT}_{[w_{\min},w_{\max}]}(\hat A/\rho_n)$. Since the hard thresholding function is $1$-Lipschitz and the entries of $W$ are between $w_{\min}$ and $w_{\max}$, we have
\begin{equation*}
  \norm{\hat W - W}_F \leq \norm{\hat A/\rho_n - W}_F
\end{equation*}
Now we decompose
\begin{equation}\label{eq:proof_usvt_1}
  \norm{\hat A/\rho_n - W}_F \leq \norm{\hat A/\rho_n - G}_F + \norm{G - W}_F
\end{equation}

To bound the first term, we observe that $\hat A$ and $G$ are both rank $\abs{S}$, so
\begin{equation*}
  \norm{\hat A/\rho_n - G}_F \leq \sqrt{2\abs{S}} \norm{\hat A/\rho_n - G}
\end{equation*}

We then decompose 
\begin{equation*}
  \norm{\hat A/\rho_n - G} \leq \frac{1}{\rho_n}\norm{\hat A - A} + \norm{A/\rho_n- W} + \norm{W - G}
\end{equation*}
By definition, $A-\hat A = \sum_{i \notin S} \sigma_i a_i a_i^\top$ so $\norm{A-\hat A} \leq \gamma\sqrt{\rho_n n}$. We have assumed that $\norm{A- \rho_n W}\leq c_r \sqrt{\rho_n n}$ holds, and moreover by Kato inequality \cite[e.g.]{Rosasco2010}:
\begin{equation*}
  \max_i \abs{\sigma_i - \rho_n \tau_i} \leq \norm{A - \rho_n W}\leq c_r \sqrt{\rho_n n}
\end{equation*}
and therefore, for all $i \notin S$
\begin{equation}\label{eq:proof_usvt_bound_tau}
  0 \leq \rho_n\tau_i \leq \sigma_i + c_r \sqrt{\rho_n n} \leq (\gamma+c_r)\sqrt{\rho_n n}\, .
\end{equation}
Thus, $\norm{W - G} = \norm{\sum_{i \notin S}\tau_i y_i y_i^\top} \leq (\gamma + c_r) \sqrt{n/\rho_n}$. At the end of the day, $\norm{\hat A/\rho_n - G} \leq 2(\gamma + c_r)\sqrt{n/\rho_n}$.

We then bound the size of the support. For $i \in S$,
\begin{equation*}
  \rho_n \tau_i \geq \sigma_i - \norm{A - \rho_n W} \geq (\gamma - c_r)\sqrt{\rho_n n}
\end{equation*}
and thus 
\begin{equation*}
  \rho_n \norm{W}_\star \geq \sum\nolimits_{i\in S} \rho_n \tau_i \geq \abs{S}(\gamma - c_r)\sqrt{\rho_n n}
\end{equation*}
and 
$
\abs{S} \leq \frac{\norm{W}_\star}{\gamma - c_r}\sqrt{\rho_n/n}
$.
At the end of the day,
\begin{equation*}
  \norm{\hat A/\rho_n - G}_F \leq 2 \sqrt{2} \frac{\gamma + c_r}{\sqrt{\gamma - c_r}} \sqrt{\norm{W}_\star}(n/\rho_n)^{1/4}
\end{equation*}

Fro the second term in \eqref{eq:proof_usvt_1}, by \eqref{eq:proof_usvt_bound_tau} we have 
\begin{align*}
  \norm{G-W}_F^2 &= \sum\nolimits_{i \notin S} \tau_i^2 \leq (\gamma+c_r)\sqrt{n/\rho_n}\sum\nolimits_{i \notin S} \tau_i \\
  &\leq (\gamma+c_r)\sqrt{n/\rho_n}\norm{W}_\star
\end{align*}
We conclude by $\norm{W}_\star = Tr(W) \leq n$.

\subsection{Proof of Theorem \ref{thm:geodesic}}\label{app:geodesic}

We start with the following Lemma.

\begin{lemma}
  Consider $\{x,y,z_1,\ldots, z_N\} \subset \Mm$ with $z_i$ iid from $\nu$. Then, for all $h_N\leq \min(h_\Mm,c_\Mm^{-1})$ and $0<\lambda_N <1/2$, if $\norm{x-y} > h_N$: with probability at least $1- e^{-c_z c_\Bb N (\lambda_N h_N)^k}$, there is $i$ such that $\norm{x-z_i} \leq h_N$ and $d(z_i, y) \leq d(x,y) - h_N(1-3\lambda_N - c_\Mm h_N)$.
\end{lemma}

\begin{proof}
  Call $\gamma_{xy} \subset \Mm$ the geodesic path between $x$ and $y$. For each point $x' \in \gamma_{xy}$, we have $d(x,y) = d(x,x') + d(x',y)$.
  Pick a point $x' \in \gamma_{xy}$ that is also on the sphere $\norm{x-x'} = h_N(1-2\lambda_N)$. Then,
  by \eqref{eq:geodesic},
  \begin{align*}
      d(x,x') &\geq \norm{x-x'} - c_\Mm \norm{x-x'}^3 \\
      &= h_N(1-2\lambda_N) (1-c_\Mm h_N^2(1-2\lambda_N)^2) \\
      &\geq h_N(1-2\lambda_N)(1-c_\Mm h_N^2)
  \end{align*}
  and therefore, since $x'\in \gamma_{xy}$,
  \begin{align*}
      d(x',y) &= d(x,y) - d(x,x') \\
      &\leq d(x,y) - h_N(1-2\lambda_N)(1-c_\Mm h_N^2)
  \end{align*}

  Now, consider the ball $B = \{z \in \Mm;~d(x',z)\leq \lambda_N h_N\}$. It has a measure $\mu(B) \geq c_\Bb (\lambda_N h_N)^k$ such that $\nu(B) \geq c_z c_\Bb (\lambda_N h_N)^k$, and with probability at least $1-(1-c_z c_\Bb(\lambda_N h_N)^k)^N \geq 1- e^{-c_z c_\Bb N (\lambda_N h_N)^k}$, there is a $z_i \in B$. Finally, since $\norm{\cdot-\cdot} \leq d(\cdot,\cdot)$, by \eqref{eq:geodesic} all points $z \in B$ satisfy
  \begin{align*}
      \norm{z-x'} &\leq d(z,x') + c_\Mm \norm{z-x'}^3 \\
      &\leq d(z,x') + c_\Mm d(z,x')^3 \\
      &\leq \lambda_N h_N + c_\Mm \lambda_N^3 h_N^3 \leq 2\lambda_N h_N
  \end{align*}
  and thus $\norm{z-x} \leq \norm{z-x'}+\norm{x-x'} \leq h_N$, and
  \begin{align*}
      d(z,y) &\leq d(z,x') + d(x',y) \\
      &\leq \lambda_N h_N + d(x,y) - h_N(1-2\lambda_N)(1-c_\Mm h_N^2) \\
      &\leq d(x,y) - h_N(1-3\lambda_N - c_\Mm h_N^2)
  \end{align*}
\end{proof}

We now prove Theorem \ref{thm:geodesic}. Consider some $x=x_i$ and $y=y_j$. Using the previous Lemma: with probability $1- e^{-c_z c_\Bb (N-n-m) (\lambda_N h_N)^k}$, there is a $z_{i_1}$ such that there is an edge $(x z_{i_1)})$ and $d(z_{i_1},y) \leq d(x,y) - h_N(1-3\lambda_N - c_\Mm h_N^2)$. Conditionally on $z_{i_1}$, we apply the same result on $z_{i_1}$ with the $N-n-m-1$ remaining points: either $\norm{z_{i_1}-y}\leq h_N$ and we take $z_{i_2}=z_{i_1}$, or with probability $1- e^{-c_z c_\Bb (N-n-m-1) (\lambda_N h_N)^k}$, there is $z_{i_2}$ connected to $z_{i_1}$ such that $d(z_{i_1},y) \leq d(x,y) - 2 h_N(1-3\lambda_N - c_\Mm h_N^2)$, so by a union bound, both the existence and $z_{i_1}$ and $z_{i_2}$ are guaranteed with probability at least $1-2e^{-c_z c_\Bb (N-n-m-1) (\lambda_N h_N)^k}$.
We repeat this process $M = \left\lfloor\frac{d(x,y)}{h_N(1-3\lambda_N - c_\Mm h_N^2)}\right\rfloor \lesssim D_\Xx/h_N$ times to obtain a path $z_{i_1},\ldots, z_{i_M}$. With probability $1-Me^{-c_z c_\Bb (N-n-m-M) (\lambda_N h_N)^k}$, either one of the $z_{i_\ell}$ is such that $\norm{z_{i_\ell}-y}\leq h_N$ and $z_{i_M} = z_{i_\ell}$, or:
\begin{align*}
  \norm{z_{i_M}-y} &\leq d(z_{i_M},y) \\
  &\leq d(x,y) - M h_N(1-3\lambda_N - c_\Mm h_N^2) \\
  &\leq h_N(1-3\lambda_N - c_\Mm h_N^2) \leq h_N
\end{align*}
We then choose $\lambda_N$ to adjust the probability to $1-\rho$, and use $\frac{1}{1-x}  = 1+\order{x}$ to obtain the upper bound on $h_N M/d(x,y)$, and therefore $h_N \textup{SP}(x,y)/d(x,y)$: since $n+m+1/h_N = o(N)$,
\begin{align*}
  \frac{h_N \textup{SP}(x,y)}{d(x,y)} &\leq \frac{h_N M}{d(x,y)} \\
  &\leq 1+ \order{c_\Mm h_N^2 + \pa{\frac{\log \frac{D_\Xx}{h_N \rho}}{N h_N^k}}^{\frac{1}{k}}}
\end{align*}

For the lower bound, we remark that the shortest path $z_{j_0} = x, z_{j_1}, \ldots, z_{j_L}=y$ (with $L\leq M$) satisfies:
\begin{align*}
  h_N \textup{SP}(x,y) &\geq \sum\nolimits_\ell \norm{z_{j_\ell} - z_{j_{\ell+1}}} \\
  &\geq \sum\nolimits_\ell d(z_{j_\ell}, z_{j_{\ell+1}}) - c_\Mm \sum\nolimits_\ell \norm{z_{j_\ell} - z_{j_{\ell+1}}}^3 \\
  &\geq d(x,y) - c_\Mm h_N^3 \textup{SP}(x,y) \\
  &\geq d(x,y) - c_\Mm h_N^3 M
\end{align*}
and we use the bound on $M$ to conclude. We finish the proof with a union bound over all pairs $(x_i,y_j)$.

\subsection{Proof of Theorem \ref{thm:fast}}\label{app:fast}

We use \eqref{eq:spectral_bound_on_K} from the proof of Theorem \ref{thm:usvt} to get
\begin{equation*}
  \abs{\Ll_{2\sigma^2,\eta}^{\hat K}(\alpha,\beta) - \Ll_{2\sigma^2,\eta}^{K}(\alpha,\beta)} \leq  \tfrac{ 2\sigma^2\eta^{2} c_\alpha c_\beta \norm{K-\hat K}}{\sqrt{nm}} 
\end{equation*}
By our choice of $\eta$, we have $\Ll_{2\sigma^2,\eta}^{K}(\alpha,\beta) = W^C_{2\sigma^2}(\alpha,\beta)$. Then, since the operator norm of any submatrix is smaller than the norm of the whole matrix
, we have
\begin{equation*}
  \norm{K - \hat K} \leq \norm{A/\rho_N - W}
\end{equation*}
Then, we use Theorem \ref{thm:Lei} to conclude.

\end{document}